\def\year{2021}\relax
\documentclass[letterpaper]{article} 
\usepackage{aaai21}  
\usepackage{times}  
\usepackage{helvet} 
\usepackage{courier}  
\usepackage[hyphens]{url}  
\usepackage{graphicx} 
\usepackage{graphicx}  
\usepackage{natbib}  
\usepackage{caption} 
\usepackage{amsmath,amssymb}
\usepackage{amsthm}
\usepackage[ruled,noline]{algorithm2e}
\usepackage{bm} 

\usepackage{tikz}
\usetikzlibrary{shapes.geometric,arrows,positioning}

\tikzstyle{process-start} = [rectangle, 
	minimum width=3.6cm, minimum height=1.2cm, text width=3.6cm,
	text centered, draw=black, 
	fill=gray!20
	]
\tikzstyle{process} = [rectangle, 
	minimum width=3.6cm, minimum height=1.2cm, text width=3.6cm,
	text centered, draw=black, 
	fill=gray!20
	]
\tikzstyle{process-long} = [rectangle, 
	minimum width=4.0cm, minimum height=1.2cm, text width=4.0cm,
	text centered, draw=black, 
	fill=gray!20
	]

\tikzstyle{algtitle} = [rectangle, 
	minimum width=3.6cm, minimum height=1cm, text width=3.6cm,
	text centered]
\tikzstyle{alg} = [rectangle, 
	minimum width=3.6cm, minimum height=1cm, text width=3.6cm,
	text centered]
\tikzstyle{alg-long} = [rectangle, 
	minimum width=4.0cm, minimum height=1cm, text width=4.0cm,
	text centered]

\tikzstyle{process-arrow} = [very thick,->,>=stealth]
\tikzstyle{alg-arrow} = [thick,dashed,->,>=stealth]

\newtheorem{theorem}{Theorem}
\newtheorem{lemma}{Lemma}
\newtheorem{definition}{Definition}
\newtheorem{assumption}{Assumption}

\title{Uncertainty-Aware Policy Optimization: \\A Robust, Adaptive Trust Region Approach}

\author {
        James Queeney,
        Ioannis Ch. Paschalidis,
        Christos G. Cassandras \\
}
\affiliations {
    Division of Systems Engineering \\
	Boston University\\
	Boston, MA 02215 \\
    \{jqueeney, yannisp, cgc\}@bu.edu
}

\begin{document}

\maketitle

\begin{abstract}
In order for reinforcement learning techniques to be useful in real-world decision making processes, they must be able to produce robust performance from limited data. Deep policy optimization methods have achieved impressive results on complex tasks, but their real-world adoption remains limited because they often require significant amounts of data to succeed. When combined with small sample sizes, these methods can result in unstable learning due to their reliance on high-dimensional sample-based estimates. In this work, we develop techniques to control the uncertainty introduced by these estimates. We leverage these techniques to propose a deep policy optimization approach designed to produce stable performance even when data is scarce. The resulting algorithm, Uncertainty-Aware Trust Region Policy Optimization, generates robust policy updates that adapt to the level of uncertainty present throughout the learning process.
\end{abstract}

\section{Introduction}\label{sec:intro}

By combining policy optimization techniques with rich function approximators such as deep neural networks, the field of reinforcement learning has achieved significant success on a variety of high-dimensional continuous control tasks \citep{duan_2016}. Despite these promising results, there are several barriers preventing the widespread adoption of deep reinforcement learning techniques for real-world decision making. Most notably, policy optimization algorithms can exhibit instability during training and high sample complexity, which are further exacerbated by the use of neural network function approximators. These are undesirable qualities in important applications such as robotics and healthcare, where data collection may be expensive and poor performance at any point can be both costly and dangerous.

Trust Region Policy Optimization (TRPO) \citep{schulman_2015} is one of the most popular methods that has been developed to address these issues, utilizing a trust region in policy space to generate stable but efficient updates. In order to perform well in practice, however, TRPO often requires a large number of samples to be collected prior to each policy update. This is because TRPO, like all policy optimization algorithms, relies on sample-based estimates to approximate expectations. These estimates are known to suffer from high variance, particularly for problems with long time horizons. As a result, the use of sample-based estimates can be a major source of error unless a sufficiently large number of samples are collected.

Motivated by the need to make decisions from limited data in real-world settings, we focus on addressing the instability caused by finite-sample estimation error in policy optimization. By directly accounting for the uncertainty present in sample-based estimates when generating policy updates, we can make efficient, robust use of limited data to produce stable performance throughout the training process. In this work, we develop an algorithm we call \emph{Uncertainty-Aware Trust Region Policy Optimization (UA-TRPO)} that controls the finite-sample estimation error in the two main components of TRPO: (i) the policy gradient and (ii) the trust region metric. Our main contributions are as follows:
\begin{enumerate}
\item We construct a finite-sample policy improvement lower bound that is accurate up to first and second order approximations, which we use to motivate an adaptive trust region that directly considers the uncertainty in the policy gradient estimate.
\item We propose a computationally efficient technique to restrict policy updates to a subspace where trust region information is available from the observed trajectories.
\item We demonstrate the robust performance of our approach through experiments on high-dimensional continuous control tasks in OpenAI Gym's MuJoCo environments \citep{brockman_2016,todorov_2012}.
\end{enumerate}
We summarize our uncertainty-aware modifications to TRPO in Figure~\ref{fig:sum}.

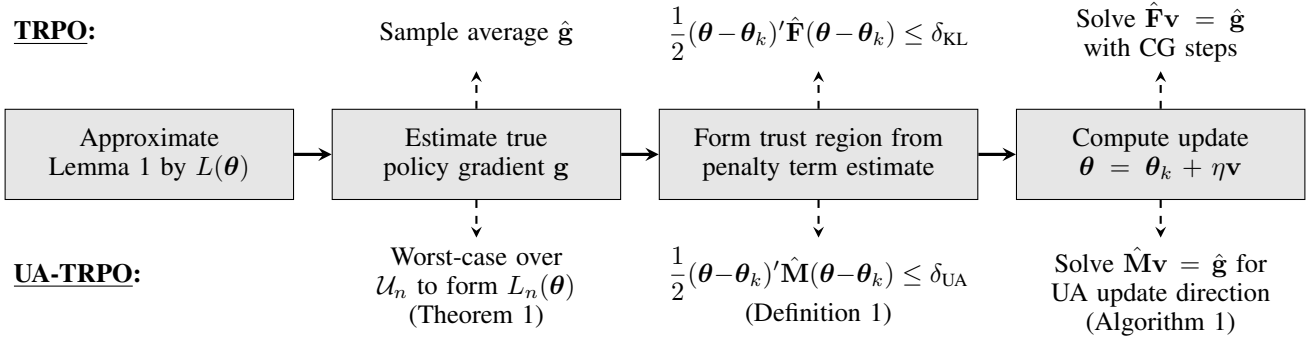
\begin{figure*}
\centering
\begin{tikzpicture}[node distance=0.5cm]
\node[process-start] (approxLB) {Approximate Lemma~\ref{lemma:pdLB} by $L(\boldsymbol\theta)$};

\node[algtitle, above=of approxLB,align=left] (title-TRPO) {\textbf{\underline{TRPO}:}};
\node[algtitle, below=of approxLB,align=left] (title-UA) {\textbf{\underline{UA-TRPO}:}};

\node[process, right=of approxLB] (PG) {Estimate true \\ policy gradient $\mathbf{g}$};

\node[alg, above=of PG] (PG-TRPO) {Sample average $\hat{\mathbf{g}}$};
\node[alg, below=of PG] (PG-UA) {Worst-case over $\mathcal{U}_n$ to form $L_n(\boldsymbol\theta)$ \\ (Theorem~\ref{thm:robLB})};

\node[process-long, right=of PG] (trust) {Form trust region from penalty term estimate};

\node[alg-long, above=of trust] (trust-TRPO) {$\displaystyle
\frac{1}{2}(\boldsymbol\theta-\boldsymbol\theta_k)' \hat{\mathbf{F}} (\boldsymbol\theta-\boldsymbol\theta_k) \leq \delta_{\text{KL}}
$
};

\node[alg-long, below=of trust] (trust-UA) {$\displaystyle
\frac{1}{2}(\boldsymbol\theta-\boldsymbol\theta_k)' \hat{\mathbf{M}} (\boldsymbol\theta-\boldsymbol\theta_k) \leq \delta_{\text{UA}}
$ \\ (Definition~\ref{def:UAtrust})
};

\node[process, right=of trust] (sol) {Compute update $\boldsymbol\theta = \boldsymbol\theta_k + \eta \mathbf{v}$};

\node[alg, above=of sol] (sol-TRPO) {Solve $\hat{\mathbf{F}}\mathbf{v} = \hat{\mathbf{g}}$ with CG steps};
\node[alg, below=of sol] (sol-UA) {Solve $\hat{\mathbf{M}}\mathbf{v} = \hat{\mathbf{g}}$ for UA update direction \\ (Algorithm~\ref{alg:update})};

\draw[process-arrow] (approxLB) -- (PG);
\draw[process-arrow] (PG) -- (trust);
\draw[process-arrow] (trust) -- (sol);

\draw[alg-arrow] (PG) -- (PG-TRPO);
\draw[alg-arrow] (PG) -- (PG-UA);
\draw[alg-arrow] (trust) -- (trust-TRPO);
\draw[alg-arrow] (trust) -- (trust-UA);
\draw[alg-arrow] (sol) -- (sol-TRPO);
\draw[alg-arrow] (sol) -- (sol-UA);

\end{tikzpicture} 
\caption{Comparison of TRPO and UA-TRPO. Both algorithms are derived from the lower bound in Lemma~\ref{lemma:pdLB} and its approximation $L(\boldsymbol\theta)$. Abbreviations: CG denotes conjugate gradient; UA denotes uncertainty-aware.}\label{fig:sum}
\end{figure*}

\section{Preliminaries}\label{sec:background}

\subsubsection{Notational Conventions.}

We use bold lowercase letters to denote vectors, bold uppercase letters to denote matrices, script letters to denote sets, and prime to denote transpose. $\mathbb{E}\left[ \, \cdot \, \right]$ represents expectation, and we use hats ($\, \hat{\cdot} \,$) to denote sample-based estimates.

\subsubsection{Reinforcement Learning Framework.}

We model the sequential decision making problem as an infinite-horizon, discounted Markov Decision Process (MDP) defined by the tuple $(\mathcal{S}, \mathcal{A}, p, r, \rho_0, \gamma)$. $\mathcal{S}$ is the set of states and $\mathcal{A}$ is the set of actions, both possibly infinite. $p: \mathcal{S} \times \mathcal{A} \rightarrow \Delta_\mathcal{S}$ is the transition probability function of the MDP where $\Delta_\mathcal{S}$ denotes the space of probability distributions over $\mathcal{S}$, $r: \mathcal{S} \times \mathcal{A} \rightarrow \mathbb{R}$ is the reward function, $\rho_0 \in \Delta_\mathcal{S}$ is the initial state distribution, and $\gamma \in [0,1)$ is the discount factor. 

We model the agent's decisions as a stationary policy $\pi: \mathcal{S} \rightarrow \Delta_\mathcal{A}$, where $\pi(a \mid s)$ is the probability of taking action $a$ in state $s$. In policy optimization, we search over a restricted class of policies $\pi_{\boldsymbol\theta} \in \Pi_{\boldsymbol\theta}$ parameterized by $\boldsymbol\theta \in \mathbb{R}^d$, where a neural network is typically used to represent $\pi_{\boldsymbol\theta}$. The standard goal in reinforcement learning is to find a policy parameter $\boldsymbol\theta$ that maximizes the expected total discounted reward $J(\boldsymbol\theta) = \mathop{\mathbb{E}}_{\tau \sim \pi_{\boldsymbol\theta}} \left[ \sum_{t=0}^\infty \gamma^t r(s_t,a_t) \right]$, where $\tau \sim \pi_{\boldsymbol\theta}$ represents a trajectory $\tau = (s_0,a_0,s_1,a_1,\ldots)$ sampled according to $s_0 \sim \rho_0(\, \cdot \,)$, $a_t \sim \pi_{\boldsymbol\theta}(\, \cdot \mid s_t)$, and $s_{t+1} \sim p(\, \cdot \mid s_t, a_t)$. 

We adopt standard reinforcement learning definitions throughout the paper. We denote the advantage function of $\pi_{\boldsymbol\theta}$ as $A_{\boldsymbol\theta}(s,a) = Q_{\boldsymbol\theta}(s,a) - V_{\boldsymbol\theta}(s)$, where $Q_{\boldsymbol\theta}(s,a) =  \mathbb{E}_{\tau \sim \pi_{\boldsymbol\theta}} \big[ \sum_{t=0}^\infty \gamma^t r(s_t,a_t) \mid s_0 = s, a_0 = a \big]$ is the state-action value function and $V_{\boldsymbol\theta}(s) = \mathbb{E}_{a \sim \pi_{\boldsymbol\theta}(\cdot \mid s)} \left[ Q_{\boldsymbol\theta}(s,a) \right]$ is the state value function. We denote the normalized discounted state visitation distribution induced by $\pi_{\boldsymbol\theta}$ as $\rho_{\boldsymbol\theta} \in \Delta_\mathcal{S}$ where $\rho_{\boldsymbol\theta}(s) = (1-\gamma) \sum_{t=0}^\infty \gamma^t P(s_t = s \mid \rho_0, \pi_{\boldsymbol\theta}, p)$, and the corresponding normalized discounted state-action visitation distribution as $d_{\boldsymbol\theta} \in \Delta_{\mathcal{S} \times \mathcal{A}}$ where $d_{\boldsymbol\theta}(s,a) = \rho_{\boldsymbol\theta}(s) \pi_{\boldsymbol\theta}(a \mid s)$.

\subsubsection{Trust Region Policy Optimization.}

TRPO is motivated by the goal of monotonic improvement used by \citet{kakade_2002} in Conservative Policy Iteration (CPI). CPI achieves monotonic improvement by maximizing a lower bound on policy improvement that can be constructed using only samples from the current policy. The lower bound developed by \citet{kakade_2002} applies only to mixture policies, but was later extended to arbitrary policies by \citet{schulman_2015} and further refined by \citet{achiam_2017}:

\begin{lemma}[\citet{achiam_2017}, Corollary 3]\label{lemma:pdLB}
Consider two policies  $\pi_{\boldsymbol\theta_k}$ and $\pi_{\boldsymbol\theta}$. Let $\epsilon_{\boldsymbol\theta} = \max_s \left| \mathbb{E}_{a \sim \pi_{\boldsymbol\theta}(\cdot \mid s)} \left[ A_{\boldsymbol\theta_k}(s,a) \right] \right|$, and denote the Kullback-Leibler (KL) divergence between $\pi_{\boldsymbol\theta_k}(\, \cdot \mid s)$ and $\pi_{\boldsymbol\theta}(\, \cdot \mid s)$ by $\text{\emph{KL}}(\boldsymbol\theta_k \Vert \boldsymbol\theta)(s)$. The difference between expected total discounted rewards $J(\boldsymbol\theta_k)$ and $J(\boldsymbol\theta)$ can be bounded below by
\begin{multline}\label{eq:LB}
J(\boldsymbol\theta) - J(\boldsymbol\theta_k) \geq \frac{1}{1-\gamma} \mathop{\mathbb{E}}_{(s,a) \sim d_{\boldsymbol\theta_k}} \left[ \frac{\pi_{\boldsymbol\theta} (a \mid s)}{\pi_{\boldsymbol\theta_k}(a \mid s)} A_{\boldsymbol\theta_k}(s,a) \right] \\ - \frac{\sqrt{2}\gamma \epsilon_{\boldsymbol\theta}}{(1-\gamma)^2} \sqrt{ \mathop{\mathbb{E}}_{s \sim \rho_{\boldsymbol\theta_k}} \left[ \text{\emph{KL}}(\boldsymbol\theta_k \Vert \boldsymbol\theta)(s) \right] },
\end{multline}	 
where the first term on the right-hand side is the surrogate objective and the second is the KL penalty term.
\end{lemma}
This lower bound can be optimized iteratively to generate a sequence of policies with parameters $\left\lbrace \boldsymbol\theta_k \right\rbrace$ and monotonically improving performance. However, this optimization problem can be difficult to solve and leads to very small policy updates in practice. Instead, TRPO introduces several modifications to this approach to produce a scalable and practical algorithm based on Lemma~\ref{lemma:pdLB}. First, TRPO considers a first order approximation of the surrogate objective and a second order approximation of the KL divergence in \eqref{eq:LB}, yielding the approximate lower bound
\begin{multline}\label{eq:approxLB}
L(\boldsymbol\theta) = \mathbf{g}'(\boldsymbol\theta-\boldsymbol\theta_k) \\-  \frac{\gamma \epsilon_{\boldsymbol\theta}}{(1-\gamma)^2} \sqrt{ (\boldsymbol\theta-\boldsymbol\theta_k)' \mathbf{F} (\boldsymbol\theta-\boldsymbol\theta_k) }, \quad
\end{multline}
with
\begin{equation}\label{eq:PG}
\begin{split}
\mathbf{g} &{} = \nabla_{\boldsymbol\theta}J(\boldsymbol\theta_k) \\ 
&{} =  \frac{1}{1-\gamma}\mathop{\mathbb{E}}_{(s,a) \sim d_{\boldsymbol\theta_k}} \left[  A_{\boldsymbol\theta_k}(s,a) \nabla_{\boldsymbol\theta} \log \pi_{\boldsymbol\theta_k}(a \mid s) \right] 
\end{split}
\end{equation}
and
\begin{equation}\label{eq:FIM}
\mathbf{F} = \mathop{\mathbb{E}}_{(s,a) \sim d_{\boldsymbol\theta_k}} \left[ \nabla_{\boldsymbol\theta} \log \pi_{\boldsymbol\theta_k}(a \mid s) \nabla_{\boldsymbol\theta} \log \pi_{\boldsymbol\theta_k}(a \mid s)' \right], 
\end{equation}
where $\mathbf{g}$ is the standard policy gradient determined by the Policy Gradient Theorem \citep{williams_1992, sutton_2000} and $\mathbf{F}$ is the average Fisher Information Matrix \citep{schulman_2015}. Note that $\mathbf{g}$ and $\mathbf{F}$ are themselves expectations, so in practice $L(\boldsymbol\theta)$ is estimated using sample averages $\hat{\mathbf{g}}$ and $\hat{\mathbf{F}}$. TRPO then reformulates the KL penalty term as a trust region constraint to produce policy updates of meaningful size, resulting in the following optimization problem at each iteration:
\begin{equation}\label{eq:TRPO}
\begin{array}{rlcl}
\displaystyle \max_{\boldsymbol\theta} & \displaystyle \hat{\mathbf{g}}'(\boldsymbol\theta-\boldsymbol\theta_k) \\
 \text{s.t.} & \displaystyle \frac{1}{2}(\boldsymbol\theta-\boldsymbol\theta_k)' \hat{\mathbf{F}} (\boldsymbol\theta-\boldsymbol\theta_k) & \leq & \displaystyle \delta_{\text{KL}},
\end{array}
\end{equation}
where the trust region parameter $\delta_{\text{KL}}$ is chosen based on the desired level of conservatism. The closed-form solution is $\boldsymbol\theta = \boldsymbol\theta_k + \eta \mathbf{v}$, where $\mathbf{v} = \hat{\mathbf{F}}^{-1}\hat{\mathbf{g}}$ is the update direction and $\eta = \sqrt{2\delta_{\text{KL}} / \mathbf{v}'\hat{\mathbf{F}}\mathbf{v}}$. The update direction cannot be calculated directly in high dimensions, so it is solved approximately by applying a finite number of conjugate gradient steps to $\hat{\mathbf{F}}\mathbf{v} = \hat{\mathbf{g}}$. Finally, a backtracking line search is performed to account for the error introduced by the first and second order approximations.  

\section{Uncertainty-Aware Trust Region}\label{sec:UAtrust}

By replacing expectations with sample-based estimates in the approximate lower bound $L(\boldsymbol\theta)$, TRPO introduces a potentially significant source of error when the number of samples $n$ used to construct the estimates is small. This finite-sample estimation error can destroy the approximate monotonic improvement argument on which TRPO is based. As a result, TRPO typically uses large amounts of data to generate stable performance.

We first address the error present in the policy gradient estimate. Rather than relying on the high-variance estimate $\hat{\mathbf{g}}$ to approximate $\mathbf{g}$ in $L(\boldsymbol\theta)$, we instead develop a robust lower bound $L_{n}(\boldsymbol\theta)$ that holds for all vectors in an uncertainty set $\mathcal{U}_n$ centered around $\hat{\mathbf{g}}$. If $\mathcal{U}_n$ contains the true policy gradient $\mathbf{g}$, $L_{n}(\boldsymbol\theta)$ will be a lower bound to $J(\boldsymbol\theta) - J(\boldsymbol\theta_k)$ up to first and second order approximation error. 

Consider the policy gradient random vector 
\begin{equation}
\boldsymbol\xi = \frac{1}{1-\gamma}A_{\boldsymbol\theta_k}(s,a) \nabla_{\boldsymbol\theta} \log \pi_{\boldsymbol\theta_k}(a \mid s) \in \mathbb{R}^d,
\end{equation}
where $(s,a) \sim d_{\boldsymbol\theta_k}$. Note that $\mathbf{g}=\mathbb{E}\left[ \boldsymbol\xi \right]$ is the true policy gradient as in \eqref{eq:PG}, and $\mathbf{\Sigma}=\mathbb{E} \left[ \left(\boldsymbol\xi-\mathbf{g}\right)\left(\boldsymbol\xi-\mathbf{g}\right)' \right]$ is the true covariance matrix of the policy gradient random vector. We make the following assumption regarding the standardized random vector $\mathbf{\Sigma}^{-1/2}(\boldsymbol\xi - \mathbf{g})$:
\begin{assumption}
$\mathbf{\Sigma}^{-1/2}(\boldsymbol\xi - \mathbf{g})$ is a sub-Gaussian random vector with variance proxy $\sigma^2$.
\end{assumption}
The sub-Gaussian assumption is a reasonable one, and is satisfied by standard assumptions in the literature such as bounded rewards and bounded $\nabla_{\boldsymbol\theta} \log \pi_{\boldsymbol\theta_k}(a \mid s)$ \citep{konda_2000,papini_2018}. Using this assumption, we construct $\mathcal{U}_n$ as follows (see the Appendix for more details):

\begin{lemma}[Constructing $\mathcal{U}_n$]\label{lemma:U}
Consider $\hat{\boldsymbol\xi}_1,\ldots,\hat{\boldsymbol\xi}_n$ independent, identically distributed random samples of $\boldsymbol\xi$, with $\hat{\mathbf{g}} = \frac{1}{n} \sum_{i=1}^n \hat{\boldsymbol\xi}_i$ their sample average. Fix $\alpha \in (0,1)$, and define 
\begin{equation}
\mathcal{U}_n = \left\lbrace \mathbf{u} \mid (\mathbf{u} - \hat{\mathbf{g}})'\mathbf{\Sigma}^{-1}(\mathbf{u} - \hat{\mathbf{g}}) \leq \sigma^2 R_n^2 \right\rbrace,
\end{equation}
where
\begin{equation}
R_n^2 = \frac{1}{n} \left( d + 2\sqrt{d \log \left( \frac{1}{\alpha} \right)} + 2 \log \left( \frac{1}{\alpha} \right) \right).
\end{equation}
Then, $\mathbf{g} \in \mathcal{U}_n$ with probability at least $1-\alpha$.
\end{lemma}

The ellipsoidal uncertainty set $\mathcal{U}_n$ constructed in Lemma~\ref{lemma:U} has an intuitive structure. It can be seen as a multivariate extension of the standard confidence interval in univariate statistics, where the radius has been calculated to accommodate the more general sub-Gaussian case \citep{hsu_2012}. We use this uncertainty set to develop a finite-sample lower bound $L_{n}(\boldsymbol\theta)$ on the performance difference between two policies:

\begin{theorem}[Finite-Sample Policy Improvement Lower Bound]\label{thm:robLB}
Consider two policies  $\pi_{\boldsymbol\theta_k}$ and $\pi_{\boldsymbol\theta}$. Let $\epsilon_{\boldsymbol\theta}$ be as defined in Lemma~\ref{lemma:pdLB}. Assume Lemma~\ref{lemma:U} holds and is used to construct $\mathcal{U}_n$ with confidence $1-\alpha$. Then, we have that
\begin{multline}\label{eq:robLB}
L_{n}(\boldsymbol\theta) = \hat{\mathbf{g}}'(\boldsymbol\theta-\boldsymbol\theta_k)  \\
{} - \frac{\gamma \epsilon_{\boldsymbol\theta}}{(1-\gamma)^2} \sqrt{ (\boldsymbol\theta-\boldsymbol\theta_k)' \mathbf{F} (\boldsymbol\theta-\boldsymbol\theta_k) } \\
{} - \sigma R_n \sqrt{(\boldsymbol\theta-\boldsymbol\theta_k)' \mathbf{\Sigma} (\boldsymbol\theta-\boldsymbol\theta_k)} \qquad
\end{multline}
is a lower bound for $J(\boldsymbol\theta) - J(\boldsymbol\theta_k)$ with probability at least $1-\alpha$, up to first and second order approximation error.
\end{theorem}

\begin{proof}
Consider a robust (i.e., worst-case with respect to $\mathcal{U}_n$) lower bound of the form 
\begin{equation}\label{eq:robLBgeneral}
\min_{\mathbf{u} \in \mathcal{U}_n} \mathbf{u}'(\boldsymbol\theta-\boldsymbol\theta_k) -  \frac{\gamma \epsilon_{\boldsymbol\theta}}{(1-\gamma)^2} \sqrt{ (\boldsymbol\theta-\boldsymbol\theta_k)' \mathbf{F} (\boldsymbol\theta-\boldsymbol\theta_k) },
\end{equation}
where $\mathcal{U}_n$ is defined as in Lemma~\ref{lemma:U}. Note that \eqref{eq:robLBgeneral} is a minimization of a linear function of $\mathbf{u}$ subject to a convex quadratic constraint in $\mathbf{u}$. By forming the Lagrangian and applying strong duality, we see that the minimum value of \eqref{eq:robLBgeneral} can be written in closed form as $L_{n}(\boldsymbol\theta)$. By construction of $\mathcal{U}_n$, $\mathbf{g}$ is a feasible solution to \eqref{eq:robLBgeneral} with probability at least $1-\alpha$. Therefore, $L(\boldsymbol\theta) \geq L_{n}(\boldsymbol\theta)$ with probability at least $1-\alpha$. By Lemma~\ref{lemma:pdLB}, $L(\boldsymbol\theta)$ is a lower bound for $J(\boldsymbol\theta) - J(\boldsymbol\theta_k)$ up to first and second order approximation error, which implies that with probability at least $1-\alpha$ so is $L_{n}(\boldsymbol\theta)$. For a more detailed proof, see the Appendix.
\end{proof}

The appearance of an additional penalty term in our robust finite-sample lower bound $L_{n}(\boldsymbol\theta)$ motivates the use of the following modified trust region:
\begin{definition}[Uncertainty-Aware Trust Region]\label{def:UAtrust}
For a given $\delta_{\text{\emph{UA}}}$, the uncertainty-aware trust region represents the set of all parameters $\boldsymbol\theta \in \mathbb{R}^d$ that satisfy the constraint
\begin{equation}\label{eq:M}
\frac{1}{2} (\boldsymbol\theta-\boldsymbol\theta_k)' \mathbf{M} (\boldsymbol\theta-\boldsymbol\theta_k) \leq \delta_{\text{\emph{UA}}},
\end{equation}
where $\mathbf{M} = \mathbf{F} + c R_n^2 \mathbf{\Sigma}, \, c \geq 0$.
\end{definition}
Note that each term in $\mathbf{M}$ accounts for a main source of potential error: the first term controls the approximation error from using on-policy expectations as in TRPO, while the second term controls the finite-sample estimation error from using the policy gradient estimate $\hat{\mathbf{g}}$. The importance of including this second term is illustrated in Figure~\ref{fig:trust}. The resulting trust region adapts to the true noise of the policy gradient random vector through $\mathbf{\Sigma}$, as well as the number of samples $n$ used to estimate the policy gradient through the coefficient $R_n^2$. We include the parameter $c \geq 0$ to control the trade-off between the two terms of $\mathbf{M}$, with $c=0$ corresponding to standard TRPO (i.e., no penalty for finite-sample estimation error).

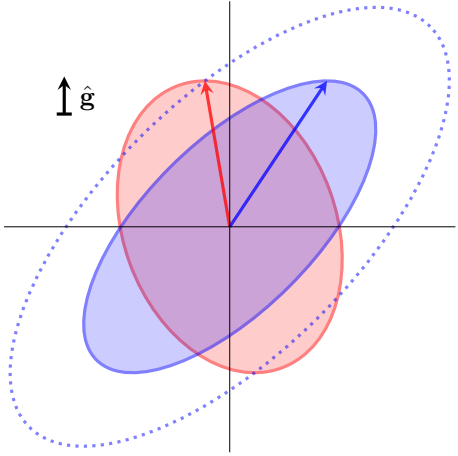
\begin{figure}
\centering
\begin{tikzpicture}
\filldraw[color=red, fill=red, draw opacity = 0.5, fill opacity = 0.2, very thick,rotate=20] (0,0) ellipse (1.41421356 and 2);
\filldraw[color=blue, fill=blue, draw opacity = 0.5, fill opacity = 0.2, very thick,rotate=45] (0,0) ellipse (2.50534244 and 1.1204232);

\filldraw[color=blue, fill=none, draw opacity = 0.5, fill opacity = 0.2, very thick,dotted,rotate=45] (0,0) ellipse (3.76807859 and 1.68513598);

\draw[color=red,very thick,opacity=0.75,->,>=stealth] (0,0) -- (-0.33122627,1.94062991);
\draw[color=blue,very thick,opacity=0.75,->,>=stealth] (0,0) -- (1.29375327,1.94062991);

\draw[color=black,very thick,->,>=stealth] (-2.2,1.5) -- (-2.2,2) node[midway,right,xshift=2,yshift=-1] {$\hat{\mathbf{g}}$};
\draw[color=black,very thick] (-2.3,1.5) -- (-2.1,1.5);

\draw[color=black,ultra thin] (-3,0) -- (3,0);
\draw[color=black,ultra thin] (0,-3) -- (0,3);
\end{tikzpicture}
\caption{Illustration of trust regions and corresponding policy updates in parameter space for TRPO (red) and UA-TRPO (blue). The dotted blue line represents the smallest uncertainty-aware trust region that contains the update proposed by TRPO. By accounting for the uncertainty present in the policy gradient estimate $\hat{\mathbf{g}}$, UA-TRPO achieves the same level of policy improvement as TRPO with lower total potential error (even though, in this case, UA-TRPO produces a larger policy update than TRPO in terms of KL divergence).}\label{fig:trust}
\end{figure}

This results in a modified policy update based on the optimization problem 
\begin{equation}\label{eq:UATRPO}
\begin{array}{rlcl}
\displaystyle \max_{\boldsymbol\theta} & \displaystyle \hat{\mathbf{g}}'(\boldsymbol\theta-\boldsymbol\theta_k) \\
 \text{s.t.} & \displaystyle \frac{1}{2}(\boldsymbol\theta-\boldsymbol\theta_k)' \hat{\mathbf{M}} (\boldsymbol\theta-\boldsymbol\theta_k) & \leq & \displaystyle \delta_{\text{UA}},
\end{array}
\end{equation}
where $\hat{\mathbf{F}}$ in \eqref{eq:TRPO} has been replaced by a sample-based estimate $\hat{\mathbf{M}} = \hat{\mathbf{F}} + c R_n^2 \hat{\mathbf{\Sigma}}$ of the uncertainty-aware trust region matrix.

\section{Uncertainty-Aware Update Direction}\label{sec:UAupdate}

The need to use the sample-based estimate $\hat{\mathbf{M}}$ for the trust region in \eqref{eq:UATRPO} introduces another potential source of error. In particular, because we are estimating a high-dimensional matrix using a limited number of samples, $\hat{\mathbf{M}}$ is unlikely to be full rank. This creates multiple problems when approximating the update direction $\mathbf{v} = \mathbf{M}^{-1}\mathbf{g}$ by solving the system of equations $\hat{\mathbf{M}}\mathbf{v}=\hat{\mathbf{g}}$. First, it is unlikely that this system of equations has an exact solution, so we must consider a least-squares solution instead. Second, the least-squares solution is not unique because the null space of $\hat{\mathbf{M}}$ contains non-zero directions. This second point is particularly important for managing uncertainty, as the null space of $\hat{\mathbf{M}}$ can be interpreted as the directions in parameter space where we have no information on the trust region metric from observed data.

In order to produce a stable, uncertainty-aware policy update, we should restrict our attention to directions in parameter space where an estimate of the trust region metric is available. Mathematically, this means we are interested in finding a least-squares solution to $\hat{\mathbf{M}}\mathbf{v}=\hat{\mathbf{g}}$ that is contained in the row space of $\hat{\mathbf{M}}$ (equivalently, the range of $\hat{\mathbf{M}}$ since the matrix is symmetric). The unique least-squares solution that satisfies this restriction is $\mathbf{v} = \hat{\mathbf{M}}^+\hat{\mathbf{g}}$, where $\hat{\mathbf{M}}^+$ denotes the Moore-Penrose pseudoinverse of $\hat{\mathbf{M}}$. It is important to note that the standard implementation of TRPO does not produce this update direction in general, leading to unstable and inefficient updates when sample sizes are small.

If $\hat{\mathbf{M}}$ has rank $p < d$, it can be written as $\hat{\mathbf{M}}=\mathbf{UDU'}$ where $\mathbf{U} \in \mathbb{R}^{d \times p}$ is an orthonormal eigenbasis for the range of $\hat{\mathbf{M}}$ and $\mathbf{D} \in \mathbb{R}^{p \times p}$ is a diagonal matrix of the corresponding positive eigenvalues. The Moore-Penrose pseudoinverse is $\hat{\mathbf{M}}^+ = \mathbf{U}\mathbf{D}^{-1}\mathbf{U}'$, and the uncertainty-aware update direction can be calculated as $\mathbf{v} =  \mathbf{U}\mathbf{D}^{-1}\mathbf{U}'\hat{\mathbf{g}}$. Intuitively, this solution is obtained by first restricting $\hat{\mathbf{M}}$ and $\hat{\mathbf{g}}$ to the $p$-dimensional subspace spanned by the basis $\mathbf{U}$, finding the unique solution to the resulting $p$-dimensional system of equations, and representing this solution in terms of its coordinates in parameter space.

Unfortunately, standard methods for computing a decomposition of $\hat{\mathbf{M}}$ are computationally intractable in high dimensions. Rather than considering the full range of $\hat{\mathbf{M}}$ spanned by $\mathbf{U}$, we propose to instead restrict policy updates to a low-rank subspace of the range using random projections \citep{halko_2011}. By generating $m \ll d$ random projections, we can efficiently calculate a basis for this subspace and the corresponding uncertainty-aware update direction (see Algorithm~\ref{alg:update} for details). Because the subspace is contained in the range of $\hat{\mathbf{M}}$ by construction, we preserve our original goal of restricting policy updates to directions in parameter space where trust region information is available. 

Our update method provides the additional benefit of generating a sufficient statistic that can be stored efficiently in memory ($\mathbf{Y}$ in Algorithm~\ref{alg:update}), which is not the case in TRPO. Because the sufficient statistic can be stored in memory, we can utilize exponential moving averages to produce more stable estimates of the trust region matrix \citep{wu_2017,kingma_2015}. This additional source of stability can be implemented through minor modifications to Algorithm~\ref{alg:update}, which we detail in the Appendix.

\begin{algorithm}[t]
\KwIn{sample-based estimates $\hat{\mathbf{g}} \in \mathbb{R}^{d}$, $\hat{\mathbf{M}} \in \mathbb{R}^{d \times d}$; random matrix $\boldsymbol\Omega \in \mathbb{R}^{d \times m}$.}
\BlankLine
Generate $m$ random projections onto the range of $\hat{\mathbf{M}}$: 
\begin{equation*}
\mathbf{Y}=\hat{\mathbf{M}}\boldsymbol\Omega.
\end{equation*} \\
Construct basis $\mathbf{Q} \in \mathbb{R}^{d \times \ell}$, $\ell \leq m$, with orthonormal vectors via the singular value decomposition of $\mathbf{Y}$. \\
\BlankLine
Compute projections of $\hat{\mathbf{M}}$, $\hat{\mathbf{g}}$ in the $\ell$-dimensional subspace spanned by $\mathbf{Q}$: 
\begin{equation*}
\widetilde{\mathbf{M}} = \mathbf{Q}'\hat{\mathbf{M}}\mathbf{Q}  \in \mathbb{R}^{\ell \times \ell}, \quad \widetilde{\mathbf{g}} = \mathbf{Q}'\hat{\mathbf{g}} \in \mathbb{R}^{\ell}.
\end{equation*} \\
Construct  the eigenvalue decomposition of $\widetilde{\mathbf{M}}$:
\begin{equation*}
\widetilde{\mathbf{M}} = \mathbf{V}\boldsymbol\Lambda\mathbf{V}'.
\end{equation*} \\
Solve the $\ell$-dimensional system $\widetilde{\mathbf{M}}\mathbf{y} = \widetilde{\mathbf{g}}$ for $\mathbf{y} \in \mathbb{R}^{\ell}$:
\begin{equation*}
\mathbf{y} = \widetilde{\mathbf{M}}^{-1}\widetilde{\mathbf{g}} = \mathbf{V}\boldsymbol\Lambda^{-1}\mathbf{V}'\mathbf{Q}'\hat{\mathbf{g}}.
\end{equation*} \\
Represent $\mathbf{y}$ in parameter space coordinates to obtain an uncertainty-aware update direction $\mathbf{v} \in \mathbb{R}^{d}$:
\begin{equation*}
\mathbf{v} = \mathbf{Q}\mathbf{y} = \mathbf{Q}\mathbf{V}\boldsymbol\Lambda^{-1}\mathbf{V}'\mathbf{Q}'\hat{\mathbf{g}}.
\end{equation*} \\
\caption{Uncertainty-Aware Update Direction}\label{alg:update}
\end{algorithm}

\section{Algorithm}\label{sec:alg}

By applying the uncertainty-aware trust region from Definition~\ref{def:UAtrust} and computing an uncertainty-aware update direction via Algorithm~\ref{alg:update}, we develop a robust policy optimization method that adapts to the uncertainty present in the sample-based estimates of both the policy gradient and the trust region metric. These important modifications to the standard trust region approach result in our algorithm Uncertainty-Aware Trust Region Policy Optimization (UA-TRPO), which is detailed in Algorithm~\ref{alg:UATRPO}.

\begin{algorithm}[t]
\KwIn{initial policy parameterization $\boldsymbol\theta_0 \in \mathbb{R}^d$;\newline trust region parameters $\delta_{\text{UA}}, c, \alpha$; \newline random matrix $\boldsymbol\Omega \in \mathbb{R}^{d \times m}$.}
\BlankLine
\For{$k=0,1,2,\ldots$}{
	\BlankLine
	Collect sample trajectories $\tau_1,\ldots,\tau_n \sim \pi_{\boldsymbol\theta_k}$.
	\BlankLine
	Calculate sample-based estimates of the policy gradient $\hat{\mathbf{g}}$ and uncertainty-aware trust region matrix $\hat{\mathbf{M}} = \hat{\mathbf{F}} + c R_n^2 \hat{\mathbf{\Sigma}}$.
	\BlankLine
	Use Algorithm~\ref{alg:update} to compute an uncertainty-aware update direction $\mathbf{v}$.
	\BlankLine
	Apply the policy update:
	\begin{equation*}
	\boldsymbol\theta_{k+1} = \boldsymbol\theta_k + \eta \mathbf{v}, \quad \eta = \sqrt{\frac{2\delta_{\text{UA}}}{ \mathbf{v}'\hat{\mathbf{M}}\mathbf{v}}}.
	\end{equation*} \\
}

\caption{Uncertainty-Aware Trust Region \\ Policy Optimization (UA-TRPO)}\label{alg:UATRPO}
\end{algorithm}

\section{Experiments}\label{sec:experiments}

In our experiments, we aim to investigate the robustness and stability of TRPO and UA-TRPO when a limited amount of data is used for each policy update. In order to accomplish this, we perform simulations on several MuJoCo environments \citep{todorov_2012} in OpenAI Gym \citep{brockman_2016}. We focus on the locomotion tasks in OpenAI Baselines' \citep{dhariwal_2017} MuJoCo1M benchmark set (Swimmer-v3, Hopper-v3, HalfCheetah-v3, and Walker2d-v3), all of which have continuous, high-dimensional state and action spaces. 

Because we are interested in evaluating the performance of TRPO and UA-TRPO when updates must be made from limited data, we perform policy updates every $1{,}000$ steps in our experiments. The tasks we consider all have a maximum time horizon of $1{,}000$, so our choice of batch size represents as little as one trajectory per policy update. Most implementations of TRPO in the literature make use of larger batch sizes, such as $5{,}000$ \citep{henderson_2018}, $25{,}000$ \citep{wu_2017}, or $50{,}000$ \citep{duan_2016} steps per policy update. We run each experiment for a total of one million steps, and we consider $50$ random seeds.

We represent our policy $\pi_{\boldsymbol\theta}$ as a multivariate Gaussian distribution, where the mean action for a given state is parameterized by a neural network with two hidden layers of 64 units each and tanh activations. The standard deviation is parameterized separately, and is independent of the state. This is a commonly used policy structure in deep reinforcement learning with continuous actions \citep{henderson_2018}. The combination of high-dimensional state and action spaces with our neural network policy representation results in policy gradients with dimension $d$ ranging between $4{,}800$ and $5{,}800$.

We use the hyperparameters from \citet{henderson_2018} for our implementation of TRPO, which includes $\delta_{\text{KL}} = 0.01$ (see Equation~\eqref{eq:TRPO}). For UA-TRPO, we use $\delta_{\text{UA}}=0.03$, $c=6\mathrm{e}{-4}$, and $\alpha=0.05$ for the inputs to Algorithm~\ref{alg:UATRPO} in all of our experiments. We determined $\delta_{\text{UA}}$ through cross validation, where the trade-off parameter $c$ was chosen so that on average the KL divergence between consecutive policies is the same as in TRPO to provide a fair comparison. See the Appendix for additional details.

\subsubsection{Robustness Comparison.}
In order to evaluate the robustness of TRPO and UA-TRPO, we consider the conditional value at risk (CVaR) of cumulative reward across our $50$ trials. For a given $\kappa \in [0,1]$, $\kappa$-CVaR represents the expected value of the bottom $\kappa$ quantiles. 

First, we consider the CVaR of final performance after one million steps of training. Figure~\ref{fig:cvar_final} displays the $\kappa$-CVaR of final performance for all $\kappa \in [0,1]$. For small values of $\kappa$ where $\kappa$-CVaR represents a more robust measure of performance, UA-TRPO is comparable to or outperforms TRPO across all environments. In all environments except Walker2d-v3, the final $\kappa$-CVaR of UA-TRPO exceeds that of TRPO for almost all values of $\kappa \in [0,1]$. 

\begin{figure}
\centering
\includegraphics[width=\columnwidth]{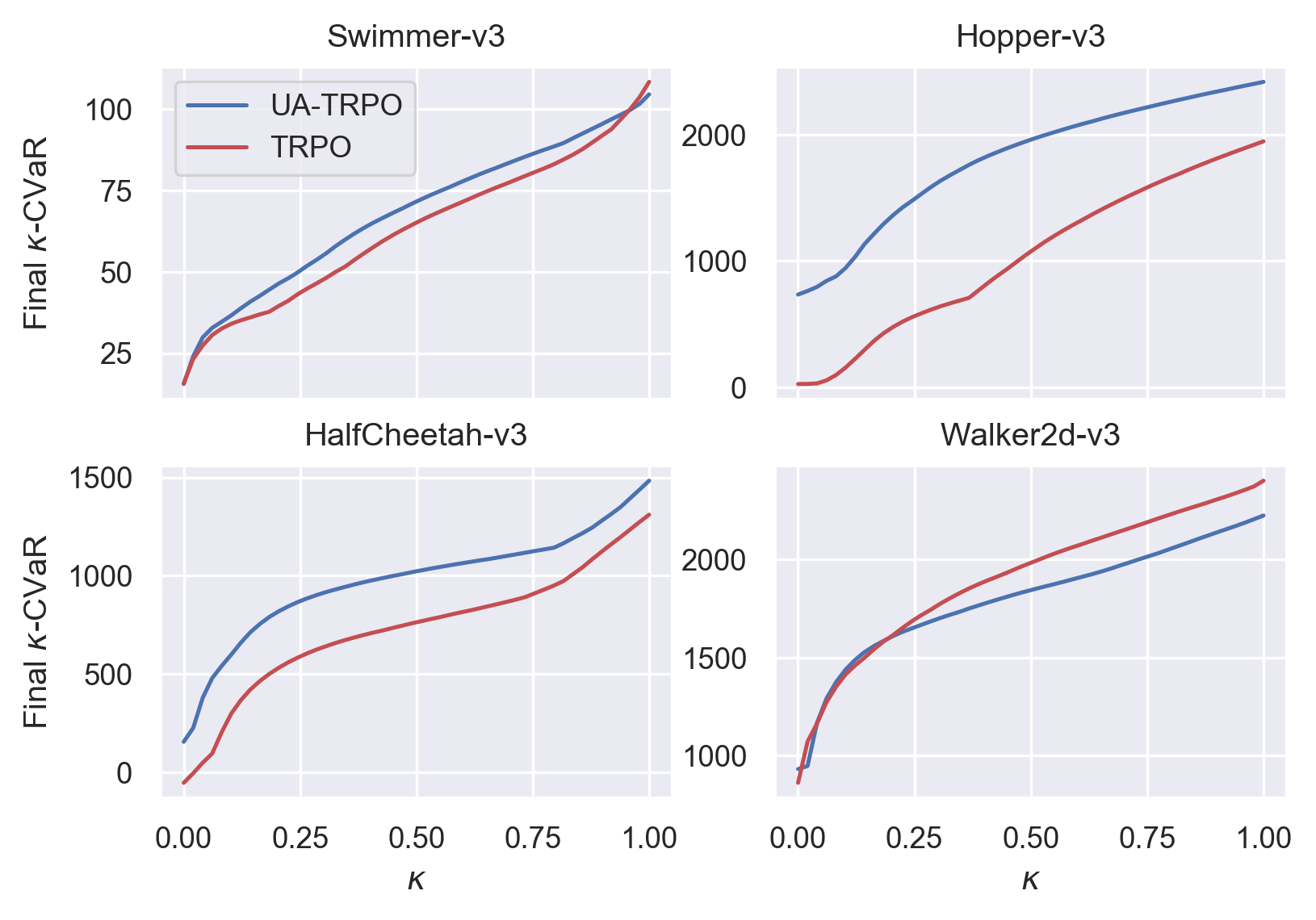}
\caption{$\kappa$-CVaR of final performance for all $\kappa \in [0,1]$.}\label{fig:cvar_final}
\end{figure}

As shown in Figure~\ref{fig:cvar_over_time}, the robustness of UA-TRPO extends beyond final performance. In all environments, UA-TRPO also demonstrates comparable or improved $20\%$-CVaR throughout the training process. In addition, we see that TRPO actually results in a decrease in $20\%$-CVaR over time in less stable environments such as Hopper-v3. Clearly, the potential for such instability would be unacceptable in a high-stakes real-world setting.

\begin{figure}
\centering
\includegraphics[width=\columnwidth]{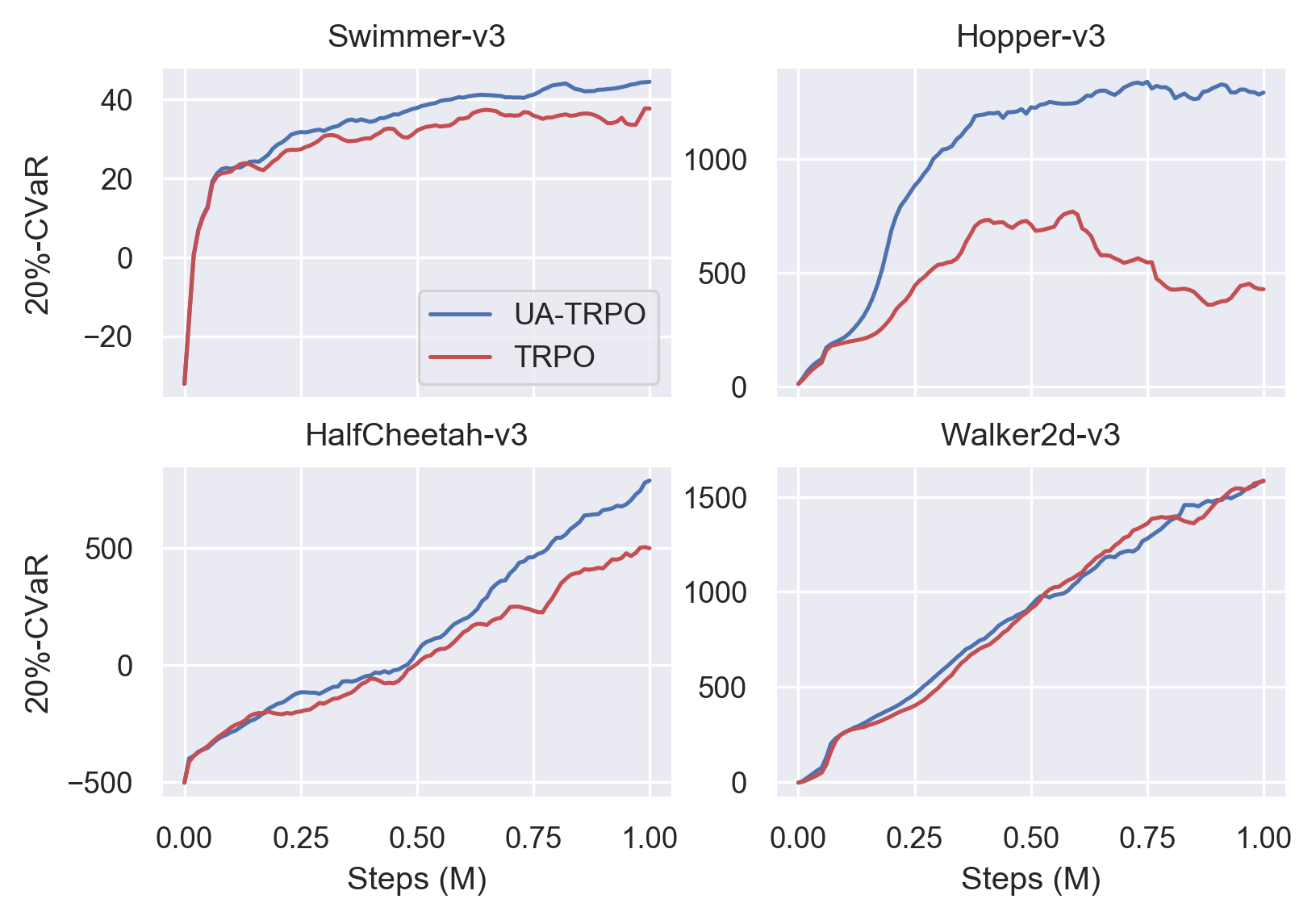}
\caption{$20\%$-CVaR of performance throughout training.}\label{fig:cvar_over_time}
\end{figure}

\subsubsection{Average Performance Comparison.}
Because UA-TRPO is a conservative approach whose primary objective is to guarantee robust policy improvement, it is possible for average performance to suffer in order to achieve robustness. However, the uncertainty-aware trust region automatically adapts to the level of uncertainty present in the problem, which prevents the algorithm from being more conservative than it needs to be. As a result, we find that UA-TRPO is able to improve robustness without sacrificing average performance. We see in Figure~\ref{fig:ave_overtime_base} that there is no statistically significant difference in average performance between TRPO and UA-TRPO in more stable tasks such as Swimmer-v3, HalfCheetah-v3, and Walker2d-v3. In less stable environments such as Hopper-v3, beyond robustness benefits, UA-TRPO also leads to a statistically significant improvement in average performance.

\begin{figure}
\centering
\includegraphics[width=\columnwidth]{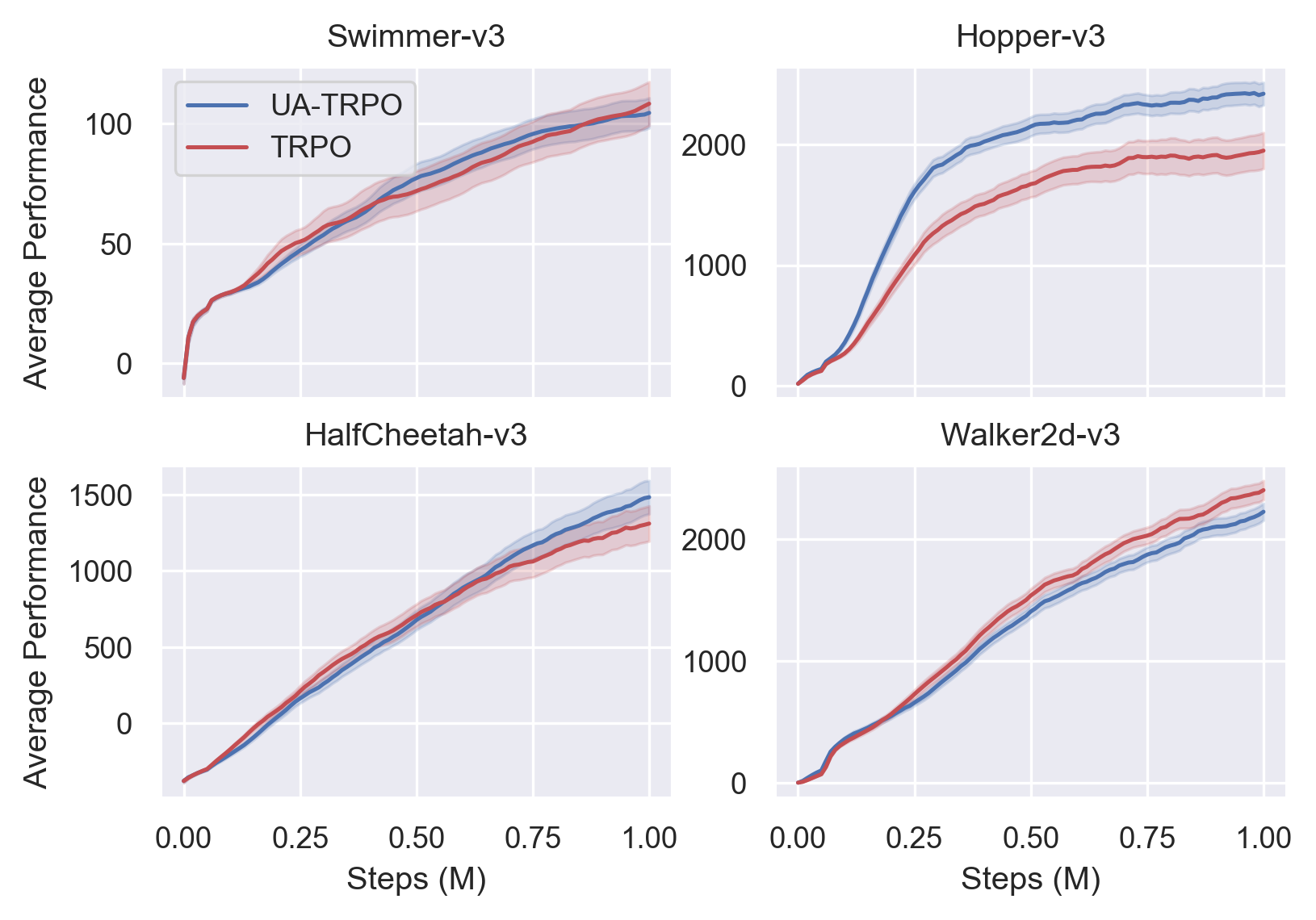}
\caption{Average performance throughout training. Shading denotes one standard error.}\label{fig:ave_overtime_base}
\end{figure}

\subsubsection{Adversarial Gradient Noise.}
We further investigate the robustness of TRPO and UA-TRPO by introducing adversarial noise to the sample-based gradient estimates used to determine policy updates. For each dimension of the gradient, we add noise in the opposite direction of the sample-based estimate. We set the magnitude of this noise to be the standard error of the policy gradient estimate in each dimension, which allows our adversarial noise to mimic the impact of a plausible level of finite-sample estimation error.

We show the impact of these adversarial gradient perturbations on average performance in Figure~\ref{fig:ave_overtime_agn100}. As expected, we see a meaningful decrease in performance compared to Figure~\ref{fig:ave_overtime_base} for both TRPO and UA-TRPO due to the adversarial nature of the noise. However, because the uncertainty-aware trust region penalizes directions with high variance, UA-TRPO is considerably more robust to this noise than TRPO. In the presence of adversarial noise, UA-TRPO demonstrates a statistically significant increase in average performance compared to TRPO for all tasks except Swimmer-v3, including improvements of 98 and 66 percent on HalfCheetah-v3 and Walker2d-v3, respectively. 

\subsubsection{Quality of Proposed Policy Updates.}
In order to better understand the performance of TRPO and UA-TRPO with small sample sizes, we analyze the quality of the policy updates proposed by these algorithms. In particular, TRPO is designed to target a specific KL divergence step size determined by $\delta_{\text{KL}}$, so we compare the actual KL divergence of the proposed policy update to the KL divergence estimated by the algorithm. Figure~\ref{fig:klratio} shows that almost $40$ percent of the updates proposed by TRPO are at least two times larger than desired, and almost $20$ percent of the updates are at least three times larger than desired. This major discrepancy between actual and estimated KL divergence is caused by the inability to produce accurate estimates of the trust region matrix from only a small number of samples. TRPO corrects this issue by using a backtracking line search, but the algorithm's strong reliance on this post-optimization process shows that the use of small batch sizes results in inefficient policy updates. 

On the other hand, we see in Figure~\ref{fig:klratio} that UA-TRPO produces policy updates in line with the algorithm's intended goal (i.e., a ratio near one). This is accomplished by using uncertainty-aware update directions to generate policy updates, which prevent the algorithm from exploiting directions in parameter space where trust region information is not available due to limited data.

\begin{figure}
\centering
\includegraphics[width=\columnwidth]{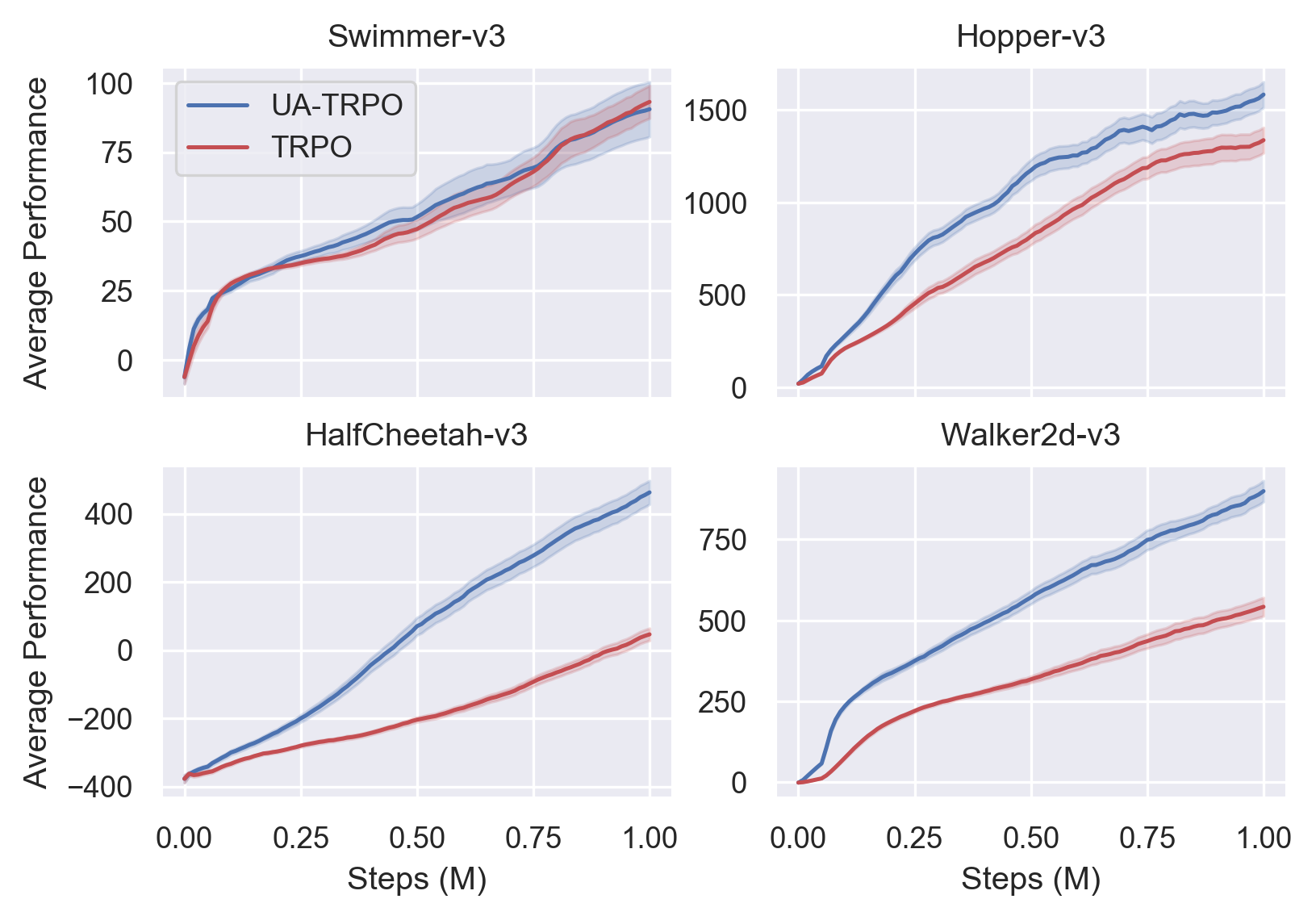}
\caption{Average performance throughout training with adversarial gradient noise. Shading denotes one standard error.}\label{fig:ave_overtime_agn100}
\end{figure}

\section{Related Work}\label{sec:related}

The difficulties of using sample-based estimates in reinforcement learning have been a topic of interest for quite some time. In particular, policy gradient methods are known to suffer from high variance. Variance reduction techniques have been proposed to mitigate this issue, such as the use of baselines and bootstrapping in advantage function estimation \citep{sutton_2000,schulman_2016}. Recently, there has been renewed interest in variance reduction in the stochastic optimization literature \citep{johnson_2013}, and some of these methods have been applied to reinforcement learning \citep{papini_2018}. This line of research is orthogonal to our work, as UA-TRPO adapts to the variance that remains after these techniques have been applied.

Conservative policy optimization approaches such as Conservative Policy Iteration \citep{kakade_2002}, TRPO \citep{schulman_2015}, and Proximal Policy Optimization \citep{schulman_2017} consider a lower bound on policy improvement to generate stable policy updates. However, these approaches rely on large batch sizes to control sample-based estimation error. In settings where access to samples may be limited, this approach to reducing estimation error may not be feasible. \citet{li_2011} developed the ``knows what it knows'' (KWIK) framework for this scenario, which allows an algorithm to choose not to produce an output when uncertainty is high. Several approaches in reinforcement learning can be viewed as applications of this uncertainty-aware framework. \citet{laroche_2019} bootstrapped the learned policy with a known baseline policy in areas of the state space where data was limited, while \citet{thomas_2015} only produced updates when the value of the new policy exceeded a baseline value with high probability. Our approach is also motivated by the KWIK framework, restricting updates to areas of the parameter space where information is available and producing updates close to the current policy when uncertainty is high. 

Other methods have considered adversarial formulations to promote stability in the presence of uncertainty. \citet{rajeswaran_2017} trained a policy with TRPO using adversarially selected trajectories from an ensemble of models, while \citet{pinto_2017} applied TRPO in the presence of an adversary that was jointly trained to destabilize learning. Because both of these methods are motivated by the goal of sim-to-real transfer, they assume that the environment can be altered during training. We introduce an adversary specifically designed to address finite-sample estimation error through a worst-case formulation over $\mathcal{U}_n$, but our method does not require any changes to be made to the underlying environment. 

Finally, our approach is related to adaptive learning rate methods such as Adam \citep{kingma_2015} that have become popular in stochastic optimization. These methods use estimates of the second moment of the gradient to adapt the step size throughout training. We accomplish this by incorporating $\mathbf{\Sigma}$ in our trust region. In fact, adaptive learning rate methods can be interpreted as using an uncentered, diagonal approximation of $\mathbf{\Sigma}$ to generate uncertainty-aware updates.  

\begin{figure}
\centering
\includegraphics[width=\columnwidth]{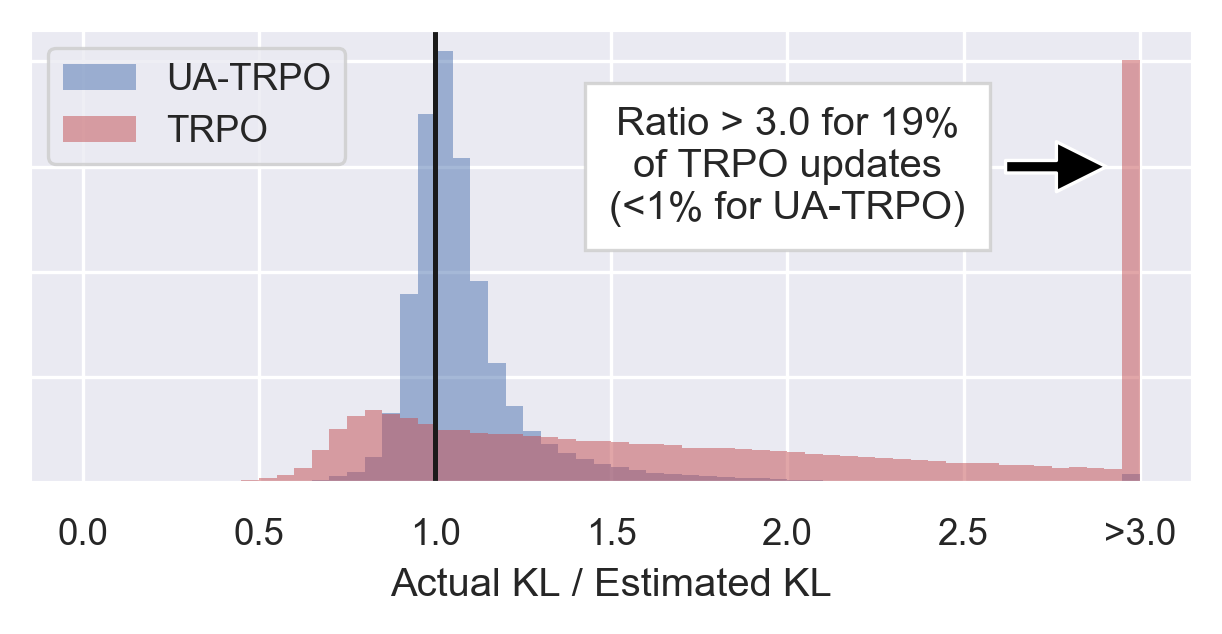}
\caption{Ratio of actual to estimated KL divergence for proposed policy updates, prior to the application of a backtracking line search. Histogram includes proposed policy updates across all environments and all random seeds.}\label{fig:klratio}
\end{figure}

\section{Conclusion}\label{sec:conc}

We have presented a principled approach to policy optimization in the presence of finite-sample estimation error. We developed techniques that adapt to the uncertainty introduced by sample-based estimates of the policy gradient and trust region metric, resulting in robust and stable updates throughout the learning process. Importantly, our algorithm, UA-TRPO, directly controls estimation error in a scalable and practical way, making it compatible with the use of rich, high-dimensional neural network policy representations. This represents an important step towards developing deep reinforcement learning methods that can be used for real-world decision making tasks where data is limited and stable performance is critical.

\section*{Ethics Statement}

Reinforcement learning has the potential to improve decision making across many important application areas such as robotics and healthcare, but these techniques will only be useful to society if they can be trusted to produce robust and stable results. Our work makes progress towards accomplishing this goal by addressing major sources of error that currently prevent the real-world adoption of policy optimization algorithms. We do not believe our contributions introduce any ethical issues that could negatively impact society.

\section*{Acknowledgments}

This research was partially supported by the NSF under grants ECCS-1931600, DMS-1664644, CNS-1645681, and IIS-1914792, by the ONR under grant N00014-19-1-2571, by the NIH under grants R01 GM135930 and UL54 TR004130, by the DOE under grant DE-AR-0001282, by AFOSR under grant FA9550-19-1-0158, by ARPA-E’s NEXTCAR program under grant DE-AR0000796, and by the MathWorks.

\bibliography{5930.QueeneyJ_bibliography}

\begin{thebibliography}{24}
\providecommand{\natexlab}[1]{#1}
\providecommand{\url}[1]{\texttt{#1}}
\providecommand{\urlprefix}{URL }
\expandafter\ifx\csname urlstyle\endcsname\relax
  \providecommand{\doi}[1]{doi:\discretionary{}{}{}#1}\else
  \providecommand{\doi}{doi:\discretionary{}{}{}\begingroup
  \urlstyle{rm}\Url}\fi

\bibitem[{Achiam et~al.(2017)Achiam, Held, Tamar, and Abbeel}]{achiam_2017}
Achiam, J.; Held, D.; Tamar, A.; and Abbeel, P. 2017.
\newblock Constrained policy optimization.
\newblock In \emph{Proceedings of the 34th International Conference on Machine
  Learning}, volume~70, 22--31. PMLR.

\bibitem[{Brockman et~al.(2016)Brockman, Cheung, Pettersson, Schneider,
  Schulman, Tang, and Zaremba}]{brockman_2016}
Brockman, G.; Cheung, V.; Pettersson, L.; Schneider, J.; Schulman, J.; Tang,
  J.; and Zaremba, W. 2016.
\newblock OpenAI Gym.
\newblock arXiv preprint.
\newblock {arXiv:1606.01540}.

\bibitem[{Dhariwal et~al.(2017)Dhariwal, Hesse, Klimov, Nichol, Plappert,
  Radford, Schulman, Sidor, Wu, and Zhokhov}]{dhariwal_2017}
Dhariwal, P.; Hesse, C.; Klimov, O.; Nichol, A.; Plappert, M.; Radford, A.;
  Schulman, J.; Sidor, S.; Wu, Y.; and Zhokhov, P. 2017.
\newblock OpenAI Baselines.
\newblock \url{https://github.com/openai/baselines}.

\bibitem[{Duan et~al.(2016)Duan, Chen, Houthooft, Schulman, and
  Abbeel}]{duan_2016}
Duan, Y.; Chen, X.; Houthooft, R.; Schulman, J.; and Abbeel, P. 2016.
\newblock Benchmarking deep reinforcement learning for continuous control.
\newblock In \emph{Proceedings of the 33rd International Conference on Machine
  Learning}, volume~48, 1329--1338. PMLR.

\bibitem[{Halko, Martinsson, and Tropp(2011)}]{halko_2011}
Halko, N.; Martinsson, P.~G.; and Tropp, J.~A. 2011.
\newblock Finding structure with randomness: Probabilistic algorithms for
  constructing approximate matrix decompositions.
\newblock \emph{SIAM Review} 53(2): 217–288.
\newblock \doi{10.1137/090771806}.

\bibitem[{Henderson et~al.(2018)Henderson, Islam, Bachman, Pineau, Precup, and
  Meger}]{henderson_2018}
Henderson, P.; Islam, R.; Bachman, P.; Pineau, J.; Precup, D.; and Meger, D.
  2018.
\newblock Deep reinforcement learning that matters.
\newblock In \emph{Proceedings of the Thirty-Second {AAAI} Conference on
  Artificial Intelligence}, 3207--3214. {AAAI} Press.

\bibitem[{Hsu, Kakade, and Zhang(2012)}]{hsu_2012}
Hsu, D.; Kakade, S.; and Zhang, T. 2012.
\newblock A tail inequality for quadratic forms of subgaussian random vectors.
\newblock \emph{Electronic Communications in Probability} 17.
\newblock \doi{10.1214/ECP.v17-2079}.

\bibitem[{Johnson and Zhang(2013)}]{johnson_2013}
Johnson, R.; and Zhang, T. 2013.
\newblock Accelerating stochastic gradient descent using predictive variance
  reduction.
\newblock In \emph{Advances in Neural Information Processing Systems 26},
  315--323. Curran Associates, Inc.

\bibitem[{Kakade and Langford(2002)}]{kakade_2002}
Kakade, S.; and Langford, J. 2002.
\newblock Approximately optimal approximate reinforcement learning.
\newblock In \emph{Proceedings of the 19th International Conference on Machine
  Learning}, 267--274. Morgan Kaufmann Publishers Inc.

\bibitem[{Kingma and Ba(2015)}]{kingma_2015}
Kingma, D.~P.; and Ba, J. 2015.
\newblock Adam: {A} method for stochastic optimization.
\newblock In \emph{3rd International Conference on Learning Representations}.

\bibitem[{Konda and Tsitsiklis(2000)}]{konda_2000}
Konda, V.~R.; and Tsitsiklis, J.~N. 2000.
\newblock Actor-critic algorithms.
\newblock In \emph{Advances in Neural Information Processing Systems 12},
  1008--1014. MIT Press.

\bibitem[{Laroche, Trichelair, and Combes(2019)}]{laroche_2019}
Laroche, R.; Trichelair, P.; and Combes, R. T.~D. 2019.
\newblock Safe policy improvement with baseline bootstrapping.
\newblock In \emph{Proceedings of the 36th International Conference on Machine
  Learning}, volume~97, 3652--3661. PMLR.

\bibitem[{Li et~al.(2011)Li, Littman, Walsh, and Strehl}]{li_2011}
Li, L.; Littman, M.~L.; Walsh, T.~J.; and Strehl, A.~L. 2011.
\newblock Knows what it knows: {A} framework for self-aware learning.
\newblock \emph{Machine Learning} 82: 399--443.
\newblock \doi{10.1007/s10994-010-5225-4}.

\bibitem[{Papini et~al.(2018)Papini, Binaghi, Canonaco, Pirotta, and
  Restelli}]{papini_2018}
Papini, M.; Binaghi, D.; Canonaco, G.; Pirotta, M.; and Restelli, M. 2018.
\newblock Stochastic variance-reduced policy gradient.
\newblock In \emph{Proceedings of the 35th International Conference on Machine
  Learning}, volume~80, 4026--4035. PMLR.

\bibitem[{Pinto et~al.(2017)Pinto, Davidson, Sukthankar, and
  Gupta}]{pinto_2017}
Pinto, L.; Davidson, J.; Sukthankar, R.; and Gupta, A. 2017.
\newblock Robust adversarial reinforcement learning.
\newblock In \emph{Proceedings of the 34th International Conference on Machine
  Learning}, volume~70, 2817--2826. PMLR.

\bibitem[{Rajeswaran et~al.(2017)Rajeswaran, Ghotra, Ravindran, and
  Levine}]{rajeswaran_2017}
Rajeswaran, A.; Ghotra, S.; Ravindran, B.; and Levine, S. 2017.
\newblock EPOpt: Learning robust neural network policies using model ensembles.
\newblock In \emph{5th International Conference on Learning Representations}.

\bibitem[{Schulman et~al.(2015)Schulman, Levine, Abbeel, Jordan, and
  Moritz}]{schulman_2015}
Schulman, J.; Levine, S.; Abbeel, P.; Jordan, M.; and Moritz, P. 2015.
\newblock Trust region policy optimization.
\newblock In \emph{Proceedings of the 32nd International Conference on Machine
  Learning}, volume~37, 1889--1897. PMLR.

\bibitem[{Schulman et~al.(2016)Schulman, Moritz, Levine, Jordan, and
  Abbeel}]{schulman_2016}
Schulman, J.; Moritz, P.; Levine, S.; Jordan, M.~I.; and Abbeel, P. 2016.
\newblock High-dimensional continuous control using generalized advantage
  estimation.
\newblock In \emph{4th International Conference on Learning Representations}.

\bibitem[{Schulman et~al.(2017)Schulman, Wolski, Dhariwal, Radford, and
  Klimov}]{schulman_2017}
Schulman, J.; Wolski, F.; Dhariwal, P.; Radford, A.; and Klimov, O. 2017.
\newblock Proximal policy optimization algorithms.
\newblock arXiv preprint.
\newblock {arXiv:1707.06347}.

\bibitem[{Sutton et~al.(2000)Sutton, McAllester, Singh, and
  Mansour}]{sutton_2000}
Sutton, R.~S.; McAllester, D.~A.; Singh, S.~P.; and Mansour, Y. 2000.
\newblock Policy gradient methods for reinforcement learning with function
  approximation.
\newblock In \emph{Advances in Neural Information Processing Systems 12},
  1057--1063. MIT Press.

\bibitem[{Thomas, Theocharous, and Ghavamzadeh(2015)}]{thomas_2015}
Thomas, P.; Theocharous, G.; and Ghavamzadeh, M. 2015.
\newblock High confidence policy improvement.
\newblock In \emph{Proceedings of the 32nd International Conference on Machine
  Learning}, volume~37, 2380--2388. PMLR.

\bibitem[{Todorov, Erez, and Tassa(2012)}]{todorov_2012}
Todorov, E.; Erez, T.; and Tassa, Y. 2012.
\newblock MuJoCo: {A} physics engine for model-based control.
\newblock In \emph{2012 {IEEE/RSJ} International Conference on Intelligent
  Robots and Systems}, 5026--5033. {IEEE}.
\newblock \doi{10.1109/IROS.2012.6386109}.

\bibitem[{Williams(1992)}]{williams_1992}
Williams, R.~J. 1992.
\newblock Simple statistical gradient-following algorithms for connectionist
  reinforcement learning.
\newblock \emph{Machine Learning} 8(3--4): 229--256.
\newblock \doi{10.1007/BF00992696}.

\bibitem[{Wu et~al.(2017)Wu, Mansimov, Grosse, Liao, and Ba}]{wu_2017}
Wu, Y.; Mansimov, E.; Grosse, R.~B.; Liao, S.; and Ba, J. 2017.
\newblock Scalable trust-region method for deep reinforcement learning using
  Kronecker-factored approximation.
\newblock In \emph{Advances in Neural Information Processing Systems 30},
  5279--5288. Curran Associates, Inc.

\end{thebibliography}

\section*{Appendix}\label{sec:appendix}

\renewcommand{\thesubsection}{\Alph{subsection}}

\subsection{Useful Definitions and Lemmas}

\begin{definition}[Sub-Gaussian Random Variable]
A random variable $\omega \in \mathbb{R}$ is sub-Gaussian with variance proxy $\sigma^2$ if $\mathbb{E} \left[ \omega \right] = 0$ and its moment generating function satisfies $\mathbb{E} \left[ \exp \left(s\omega\right) \right] \leq \exp \left( s^2\sigma^2/2 \right)$ for all $s \in \mathbb{R}$. We denote this by $\omega \sim \emph{subG}(\sigma^2)$.
\end{definition}

\begin{definition}[Sub-Gaussian Random Vector]
A random vector $\boldsymbol\omega \in \mathbb{R}^d$ is sub-Gaussian with variance proxy $\sigma^2$ if $\mathbf{s}'\boldsymbol\omega \sim \emph{subG}(\sigma^2)$ for all $\mathbf{s} \in \mathcal{S}^{d}$, where $\mathcal{S}^{d} = \left\lbrace \mathbf{s} \in \mathbb{R}^d \mid  \Vert \mathbf{s} \Vert = 1 \right\rbrace$ is the unit sphere in $\mathbb{R}^d$. We denote this by $\boldsymbol\omega \sim \emph{subG}_d(\sigma^2)$.
\end{definition}

\begin{lemma}[Squared Norm Sub-Gaussian Concentration Inequality]\label{lemma:squarenormCI}
Let $\boldsymbol\omega \in \mathbb{R}^d$ be a sub-Gaussian random vector with variance proxy $\sigma^2$. Consider $\hat{\boldsymbol\omega}_1,\ldots,\hat{\boldsymbol\omega}_n$ independent, identically distributed random samples of $\boldsymbol\omega$, with $\hat{\mathbf{x}} = \frac{1}{n} \sum_{i=1}^n \hat{\boldsymbol\omega}_i$ their sample average. Fix $\alpha \in (0,1)$. Then,
\begin{equation}
P\left(\Vert \hat{\mathbf{x}} \Vert_2^2 \leq \sigma^2 R_n^2 \right) > 1- \alpha,
\end{equation}
where
\begin{equation}
R_n^2 = \frac{1}{n} \left( d + 2\sqrt{d \log \left( \frac{1}{\alpha} \right)} + 2 \log \left( \frac{1}{\alpha} \right) \right).
\end{equation}
\end{lemma}

\begin{proof}
Consider the random vector $\mathbf{x} = \frac{1}{n} \sum_{i=1}^n \boldsymbol\omega_i$, where $\boldsymbol\omega_i$, $i=1,\ldots,n$, are independent, identically distributed copies of the random vector $\boldsymbol\omega$. Because $\boldsymbol\omega_i \sim \text{subG}_d(\sigma^2)$, we have that $\mathbf{x} \sim \text{subG}_d(\sigma^2/n)$. Then, the result immediately follows as a special case of Theorem~2.1 in \citet{hsu_2012} applied to the sample $\hat{\mathbf{x}}$ of the random vector $\mathbf{x}$. 
\end{proof}

\subsection{Proof of Lemma~\ref{lemma:U}}

\begin{proof}
Define $\boldsymbol\omega = \mathbf{\Sigma}^{-1/2}(\boldsymbol\xi - \mathbf{g})$. Note that $\boldsymbol\omega \sim \text{subG}_d(\sigma^2)$ by assumption. Therefore, by applying Lemma~\ref{lemma:squarenormCI} with $\hat{\mathbf{x}} = \mathbf{\Sigma}^{-1/2}(\hat{\mathbf{g}} - \mathbf{g})$, we have that 
\begin{equation}
P\left( (\hat{\mathbf{g}} - \mathbf{g})'\mathbf{\Sigma}^{-1}(\hat{\mathbf{g}} - \mathbf{g}) \leq \sigma^2 R_n^2 \right) > 1 - \alpha. 
\end{equation}
This implies $\mathbf{g} \in \mathcal{U}_n$ with probability at least $1-\alpha$.
\end{proof}

\subsection{Detailed Proof of Theorem~\ref{thm:robLB}}

\begin{proof}
Consider the function
\begin{multline}
f(\mathbf{u}) = \mathbf{u}'(\boldsymbol\theta-\boldsymbol\theta_k) \\ - \frac{\gamma \epsilon_{\boldsymbol\theta}}{(1-\gamma)^2} \sqrt{ (\boldsymbol\theta-\boldsymbol\theta_k)' \mathbf{F} (\boldsymbol\theta-\boldsymbol\theta_k) }. \quad
\end{multline}
Then, the robust (i.e., worst-case with respect to $\mathcal{U}_n$) lower bound \eqref{eq:robLBgeneral} can be written as $\min_{\mathbf{u} \in \mathcal{U}_n} f(\mathbf{u})$, where $\mathcal{U}_n$ is defined as in Lemma~\ref{lemma:U}. 

We first show that the robust lower bound \eqref{eq:robLBgeneral} is equivalent to $L_{n}(\boldsymbol\theta)$ in \eqref{eq:robLB}. Note that the second term in the objective $f(\mathbf{u})$ does not depend on $\mathbf{u}$, so we can ignore it when solving the minimization problem over $\mathcal{U}_n$. Omitting this term and using the definition of $\mathcal{U}_n$ from Lemma~\ref{lemma:U}, the resulting minimization problem can be written as
\begin{equation}\label{eq:robLBdetailed}
\begin{array}{rl}
\displaystyle \min_{\mathbf{u}} & \displaystyle  \mathbf{u}'(\boldsymbol\theta-\boldsymbol\theta_k) \\
\text{s.t.} & (\mathbf{u} - \hat{\mathbf{g}})'\mathbf{\Sigma}^{-1}(\mathbf{u} - \hat{\mathbf{g}}) \leq \sigma^2 R_n^2.
\end{array}
\end{equation}
This is a minimization of a linear function subject to a convex quadratic constraint. The Lagrangian corresponding to \eqref{eq:robLBdetailed} can be written
\begin{multline}\label{eq:lagrangian}
G(\mathbf{u},\nu) = \mathbf{u}'(\boldsymbol\theta-\boldsymbol\theta_k) \\ + \nu \left[ (\mathbf{u} - \hat{\mathbf{g}})'\mathbf{\Sigma}^{-1}(\mathbf{u} - \hat{\mathbf{g}}) - \sigma^2 R_n^2 \right],
\end{multline}
where $\nu \geq 0$ is the Lagrange multiplier associated with the constraint. We apply sufficient conditions to \eqref{eq:lagrangian} to find that the Lagrangian is minimized at
\begin{equation}
\mathbf{u} = \hat{\mathbf{g}} - \frac{1}{2\nu} \mathbf{\Sigma}(\boldsymbol\theta-\boldsymbol\theta_k).
\end{equation}
By plugging this value back into \eqref{eq:lagrangian}, the dual function $D(\nu) = \min_{\mathbf{u}} G(\mathbf{u},\nu)$ can be written in closed form as
\begin{multline}\label{eq:dualfunc}
D(\nu) = \hat{\mathbf{g}}'(\boldsymbol\theta-\boldsymbol\theta_k) \\ - \frac{1}{4\nu} (\boldsymbol\theta-\boldsymbol\theta_k)'\mathbf{\Sigma}(\boldsymbol\theta-\boldsymbol\theta_k) - \nu \sigma^2 R_n^2.
\end{multline}
The corresponding dual problem is $\max_{\nu \geq 0} D(\nu)$. $D(\nu)$ is concave in $\nu$ for $\nu>0$, so we can apply sufficient conditions to find the solution to the dual problem. $D(\nu)$ is maximized at 
\begin{equation}
\nu = \frac{1}{2 \sigma R_n} \sqrt{(\boldsymbol\theta-\boldsymbol\theta_k)'\mathbf{\Sigma}(\boldsymbol\theta-\boldsymbol\theta_k)},
\end{equation}
which results in the following optimal value of the dual problem:
\begin{multline}
\max_{\nu \geq 0} D(\nu) = \hat{\mathbf{g}}'(\boldsymbol\theta-\boldsymbol\theta_k) \\ -  \sigma R_n \sqrt{(\boldsymbol\theta-\boldsymbol\theta_k)'\mathbf{\Sigma}(\boldsymbol\theta-\boldsymbol\theta_k)}. \quad
\end{multline}
By strong duality, this is the optimal value of the primal problem \eqref{eq:robLBdetailed}. This implies that the robust lower bound \eqref{eq:robLBgeneral} is equivalent to
\begin{multline}
L_{n}(\boldsymbol\theta) = \hat{\mathbf{g}}'(\boldsymbol\theta-\boldsymbol\theta_k)  \\
{} - \frac{\gamma \epsilon_{\boldsymbol\theta}}{(1-\gamma)^2} \sqrt{ (\boldsymbol\theta-\boldsymbol\theta_k)' \mathbf{F} (\boldsymbol\theta-\boldsymbol\theta_k) } \\
{} - \sigma R_n \sqrt{(\boldsymbol\theta-\boldsymbol\theta_k)' \mathbf{\Sigma} (\boldsymbol\theta-\boldsymbol\theta_k)}, \quad
\end{multline}
where we have included the second term of $f(\mathbf{u})$ that we omitted to begin the proof.

We now show that $L_{n}(\boldsymbol\theta)$ is a lower bound for $J(\boldsymbol\theta) - J(\boldsymbol\theta_k)$ with probability at least $1-\alpha$, up to first and second order approximation error. By Lemma~\ref{lemma:U}, $\mathbf{g} \in \mathcal{U}_n$ with probability at least $1-\alpha$. If $\mathbf{g} \in \mathcal{U}_n$, we have that $f(\mathbf{g}) \geq \min_{\mathbf{u} \in \mathcal{U}_n} f(\mathbf{u})$. Note that $f(\mathbf{g}) = L(\boldsymbol\theta)$ where $L(\boldsymbol\theta)$ is the approximate lower bound defined in \eqref{eq:approxLB}, and $\min_{\mathbf{u} \in \mathcal{U}_n} f(\mathbf{u}) = L_n(\boldsymbol\theta)$ as shown above. Therefore, $L(\boldsymbol\theta) \geq L_{n}(\boldsymbol\theta)$ with probability at least $1-\alpha$. By Lemma~\ref{lemma:pdLB}, $L(\boldsymbol\theta)$ is a lower bound for $J(\boldsymbol\theta) - J(\boldsymbol\theta_k)$ up to first and second order approximation error, which implies that with probability at least $1-\alpha$ so is $L_{n}(\boldsymbol\theta)$. 
\end{proof}

\subsection{Moving Average Trust Region Estimate}

When applying trust region methods with neural network policy representations, it is typically not practical to store the trust region matrix in memory due to its size. Instead, the trust region matrix is accessed through matrix-vector products. In TRPO, we use matrix-vector products when applying conjugate gradient steps to determine the update direction. In UA-TRPO, we use matrix-vector products to calculate random projections $\mathbf{Y} \in \mathbb{R}^{d \times m}$. 

Because we utilize $m \ll d$ random projections, the statistic $\mathbf{Y}$ can be stored efficiently in memory. As a result, we can maintain an exponential moving average (EMA) of this statistic that can be used to produce a more accurate low-rank approximation of the uncertainty-aware trust region matrix $\mathbf{M}$. By using exponential moving averages, we leverage recent estimates of $\mathbf{M}$ to inform the current estimate. UA-TRPO restricts the policy from changing too much between iterations, so it is reasonable to believe that recent estimates contain useful information. This can be interpreted as using a larger sample size to estimate $\mathbf{M}$ \citep{wu_2017}, which can produce a much more accurate estimate when batch sizes are small.

We can implement an EMA estimate of $\mathbf{M}$ with weight parameter $\beta$, which we denote by $\hat{\mathbf{M}}_{\beta}$, through minor modifications to Algorithm~\ref{alg:update} that we now describe.

\subsubsection*{Exponential Moving Average Statistic.}
First, we maintain two separate EMA statistics $\mathbf{Y}_{\mathbf{F}}^{(k)}$ and $\mathbf{Y}_{\mathbf{\Sigma}}^{(k)}$ for $\mathbf{F}$ and $\mathbf{\Sigma}$, respectively, because the coefficient $cR_n^2$ that controls the trade-off between $\mathbf{F}$ and $\mathbf{\Sigma}$ when forming $\mathbf{M}$ may change from iteration to iteration. We initialize these statistics at $\mathbf{Y}_{\mathbf{F}}^{(0)} = \mathbf{Y}_{\mathbf{\Sigma}}^{(0)} = \mathbf{0}$. For each policy update, we generate random projections $\mathbf{Y}_{\mathbf{F}}$ and $\mathbf{Y}_{\mathbf{\Sigma}}$ based on $\hat{\mathbf{F}}$ and $\hat{\mathbf{\Sigma}}$, respectively, which are estimated using data from the current policy. We use these random projections and the EMA weight parameter $\beta$ to update our EMA statistics. Finally, we combine these statistics using the coefficient $cR_n^2$ calculated based on the current batch, and apply a bias correction to account for initializing our statistics at $\mathbf{0}$ \citep{kingma_2015}. See Algorithm~\ref{alg:EMAstat} for details. 

This results in a statistic $\mathbf{Y}$ that represents random projections onto the range of the EMA estimate $\hat{\mathbf{M}}_{\beta}$. Algorithm~\ref{alg:update}, on the other hand, computes random projections onto the range of the estimate $\hat{\mathbf{M}}$ that is calculated using only samples from the current policy.

\begin{algorithm}[t]
\KwIn{sample-based estimates $\hat{\mathbf{F}}, \hat{\mathbf{\Sigma}} \in \mathbb{R}^{d \times d}$; random matrix $\boldsymbol\Omega \in \mathbb{R}^{d \times m}$;\newline EMA statistics $\mathbf{Y}_{\mathbf{F}}^{(k)}, \mathbf{Y}_{\mathbf{\Sigma}}^{(k)} \in \mathbb{R}^{d \times m}$; \newline EMA weight parameter $\beta$.}
\BlankLine
Generate $m$ random projections onto the ranges of $\hat{\mathbf{F}}$ and $\hat{\mathbf{\Sigma}}$: 
\begin{equation*}
\mathbf{Y}_{\mathbf{F}}=\hat{\mathbf{F}}\boldsymbol\Omega, \qquad \mathbf{Y}_{\mathbf{\Sigma}}=\hat{\mathbf{\Sigma}}\boldsymbol\Omega.
\end{equation*} \\
Update EMA statistics: 
\begin{equation*}
\begin{split}
\mathbf{Y}_{\mathbf{F}}^{(k+1)} &{} = \beta \mathbf{Y}_{\mathbf{F}}^{(k)} + (1-\beta) \mathbf{Y}_{\mathbf{F}},  \\
\mathbf{Y}_{\mathbf{\Sigma}}^{(k+1)} &{} = \beta \mathbf{Y}_{\mathbf{\Sigma}}^{(k)} + (1-\beta) \mathbf{Y}_{\mathbf{\Sigma}}.
\end{split}
\end{equation*} \\
Combine EMA statistics and apply bias correction to represent random projections onto the range of the EMA estimate $\hat{\mathbf{M}}_{\beta}$:
\begin{equation*}
\mathbf{Y} = \frac{1}{1-\beta^{k+1}} \left( \mathbf{Y}_{\mathbf{F}}^{(k+1)} + c R_n^2 \mathbf{Y}_{\mathbf{\Sigma}}^{(k+1)} \right).
\end{equation*} \\
\caption{Exponential Moving Average Statistic}\label{alg:EMAstat}
\end{algorithm}

\subsubsection*{Projection of Trust Region Matrix.}
With this modified construction of $\mathbf{Y}$, the only additional modification required to Algorithm~\ref{alg:update} concerns the projection of the trust region matrix in the low-rank subspace spanned by $\mathbf{Q}$. This step requires access to the estimate of the trust region matrix to compute 
\begin{equation}\label{eq:Mproj}
\widetilde{\mathbf{M}}_{\beta} =  \mathbf{Q}'\hat{\mathbf{M}}_{\beta}\mathbf{Q},
\end{equation}
but we only have access to $\hat{\mathbf{M}}_{\beta}$ through $\mathbf{Y}$. Instead, we can formulate a least-squares problem to approximate $\widetilde{\mathbf{M}}_{\beta}$. By multiplying each side of \eqref{eq:Mproj} by $\mathbf{Q}'\boldsymbol\Omega$ and leveraging the relationship $\hat{\mathbf{M}}_{\beta} \approx \hat{\mathbf{M}}_{\beta}\mathbf{Q}\mathbf{Q}'$ \citep{halko_2011}, we see that

\begin{equation}
\begin{split}
\widetilde{\mathbf{M}}_{\beta}\mathbf{Q}'\boldsymbol\Omega &{} =  \mathbf{Q}'\hat{\mathbf{M}}_{\beta}\mathbf{Q}\mathbf{Q}'\boldsymbol\Omega \\
&{} \approx  \mathbf{Q}'\hat{\mathbf{M}}_{\beta}\boldsymbol\Omega \\
&{} =  \mathbf{Q}'\mathbf{Y}.
\end{split}
\end{equation}
Therefore, we can approximate the projection of $\hat{\mathbf{M}}_{\beta}$ in the subspace spanned by $\mathbf{Q}$ by solving the least-squares problem $\widetilde{\mathbf{M}}_{\beta}\mathbf{Q}'\boldsymbol\Omega = \mathbf{Q}'\mathbf{Y}$ for $\widetilde{\mathbf{M}}_{\beta}$. All other steps in Algorithm~\ref{alg:update} to compute an uncertainty-aware update direction remain unchanged. 

\subsection{Implementation Details}

To aid in reproducibility, we describe additional implementation details not discussed in the Experiments section. Note that all of these choices are based on common implementations in the literature. 

The value function $V_{\boldsymbol\theta}(s)$ is parameterized by a neural network with two hidden layers of 64 units each and tanh activation functions. The advantage values $A_{\boldsymbol\theta}(s,a)$ needed to compute policy gradient estimates are calculated using Generalized Advantage Estimation (GAE) \citep{schulman_2016}, and are standardized within each batch. Observations are standardized using a running mean and standard deviation throughout the training process.

The hyperparameters used to produce the results found in the Experiments section are included in Table~\ref{tab:hyperparam}. We follow the hyperparameter choices of \citet{henderson_2018} for our implementation of TRPO, which includes the hyperparameters listed in the General and TRPO sections of Table~\ref{tab:hyperparam}. For UA-TRPO, we performed cross validation to determine the trust region parameter $\delta_{\text{UA}}$. We considered $\delta_{\text{UA}}=0.02,0.03,0.04$, and we set the trade-off parameter $c$ for each choice of $\delta_{\text{UA}}$ so that on average the KL divergence between consecutive policies was the same as in TRPO ($\delta_{\text{KL}} = 0.01$). This resulted in trade-off parameters $c=2\mathrm{e}{-4},6\mathrm{e}{-4},1\mathrm{e}{-3}$, respectively. The hyperparameters $\delta_{\text{UA}}=0.03, c=6\mathrm{e}{-4}$ were selected based on the criterion of average performance across one million steps.

\begin{table}
\centering
\begin{tabular}{lr}
\hline
 \\
\textbf{\underline{General}} \\ [1ex]
Discount Rate ($\gamma$) & $0.995$ \\ [0.5ex]
GAE Parameter ($\lambda$) & $0.97$ \\ [0.5ex]
Value Function Optimizer & Adam \\ [0.5ex]
Value Function Step Size & $0.001$ \\ [0.5ex]
Value Function Iterations per Update & $5$ \\ [0.5ex]
 \\
\textbf{\underline{TRPO}} \\ [1ex]
Trust Region Parameter ($\delta_{\text{KL}}$) & $0.01$ \\ [0.5ex]
Conjugate Gradient Iterations per Update & $20$ \\ [0.5ex]
Conjugate Gradient Damping Coefficient & $0.1$ \\ [0.5ex]
 \\
\textbf{\underline{UA-TRPO}} \\ [1ex]
Trust Region Parameter ($\delta_{\text{UA}}$) & $0.03$ \\ [0.5ex]
Trade-Off Parameter ($c$) & $6\mathrm{e}{-4}$ \\ [0.5ex]
Confidence Parameter ($\alpha$) & $0.05$ \\ [0.5ex]
Number of Random Projections ($m$) & $200$ \\ [0.5ex]
EMA Weight Parameter ($\beta$) & $0.9$ \\
 \\
\hline
\end{tabular}
\caption{Hyperparameters used in Experiments section.}\label{tab:hyperparam}
\end{table}

In order to calculate estimates of the trust region matrices, we follow the subsampling procedure described in \citet{schulman_2015} with a subsampling factor of $10$. This process leverages the definition of trust region matrices as expectations with respect to $(s,a) \sim d_{\boldsymbol\theta_k}$ to define samples as simulation steps rather than trajectories, and performs subsampling on the simulation steps in the current batch. The use of simulation steps as samples allows estimates to be based on larger sample sizes, and \citet{schulman_2015} motivates the use of subsampling for its computational benefit. However, subsampling is also important because it results in samples $(s,a) \sim d_{\boldsymbol\theta_k}$ that are only weakly dependent, which is more consistent with the standard assumption of independent samples when constructing sample-based estimates. 

We ignore the constant $1/(1-\gamma)$ factor of the policy gradient when estimating the covariance matrix, as the impact of this constant can be absorbed into the trade-off parameter. This results in estimates $\hat{\mathbf{F}}$ and $\hat{\mathbf{\Sigma}}$ whose entries are of the same order of magnitude. Finally, we use independent standard Gaussian samples to construct the random matrix $\boldsymbol\Omega$ needed to compute random projections in UA-TRPO \citep{halko_2011}.

\end{document}